\documentclass[letterpaper, 10 pt, conference]{ieeeconf}  

\IEEEoverridecommandlockouts                              

\usepackage{amsmath,amsfonts,amsthm}
\usepackage[noend]{algorithmic}
\usepackage{algorithm}
\usepackage{array}
\usepackage{textcomp}
\usepackage{stfloats}
\usepackage{url}
\usepackage{verbatim}
\usepackage{graphicx}
\usepackage{cite}
\usepackage{bbm}
\usepackage{siunitx}
\usepackage{booktabs}
\newtheorem{definition}{Definition}

\newtheorem{theorem}{Theorem}
\newtheorem{assumption}{Assumption}
\newtheorem{remark}{Remark}
\newtheorem{proposition}{Proposition}

\newboolean{redactAuthors}
\newcommand{\redactReplace}[2]{\ifthenelse{\boolean{redactAuthors}}{\textcolor{red}{#1}}{#2}}
\newcommand{\redactNamed}[2]{\ifthenelse{\boolean{redactAuthors}}{\textcolor{red}{#1 removed for double-anonymous review process}}{#2}}
\newcommand{\redactInline}[1]{\ifthenelse{\boolean{redactAuthors}}{\textcolor{red}{(removed for double-anonymous review process)}}{#1}}
\newcommand{\redactRemove}[1]{\ifthenelse{\boolean{redactAuthors}}{}{#1}}

\setboolean{redactAuthors}{false}

\title{\LARGE \bf
Distribution-Free Risk-Aware Planning and Control Under Uncertainty Using Conformal Spectral Risk Control
}

\author{
\thanks{\textcolor{red}{This work has been submitted to the IEEE for possible publication. Copyright may be transferred without notice, after which this version may no longer be accessible.}}
\redactNamed{Authorship}{Junsik Eom, Tulga Ersal$^*$}%
\redactReplace{\thanks{\textcolor{red}{Affiliation notes removed for double-anonymous review process.}}}{\thanks{J. Eom and T. Ersal are with the Department of Mechanical Engineering, University of Michigan, Ann Arbor, MI 48109. (email: \{jeom, tersal\}@umich.edu)}}%
\thanks{* Corresponding author}
}

\begin{document}

\maketitle
\thispagestyle{empty}
\pagestyle{empty}

\begin{abstract}
Safe navigation in dynamic and uncertain environments often relies on accurate estimation of, or assumptions about, the true underlying uncertainty. However, accurately characterizing the true uncertainty distribution is often difficult due to limited data or imperfect information. An incorrect understanding of the uncertainty and its associated risk may lead to dangerous decisions even under high levels of risk aversion. To address this issue, we propose a risk-aware model predictive control (RA-MPC) framework that incorporates prediction sets to guarantee risk control below a user-specified threshold without requiring assumptions about the underlying uncertainty distribution. To generate the prediction sets, we develop a distribution-free risk quantification framework that extends conformal risk control (CRC) to general spectral risk measures. We then show that incorporating the prediction sets into the MPC framework provides statistical safety guarantees in terms of spectral risk constraint satisfaction even under uncertainty misspecification. We validate the proposed framework in simulated vehicle obstacle avoidance scenarios, demonstrating improved safety and reduced solve time compared to a baseline RA-MPC framework.
\end{abstract}

\section{Introduction}

Safe planning and control for autonomous systems is gaining increasing attention as such systems are deployed in complex real-world environments, including dense urban settings \cite{kummerle2015autonomous} and off-road terrain \cite{yin2023reliable}. 
A major challenge in these environments is the presence of uncertainties arising from imperfect information about the surrounding environment. 
To ensure safety in uncertain environments, autonomous systems must be able to systematically evaluate the risks associated with uncertainty \cite{akella2025risk}.

Among many methods for safe planning and control, model predictive control (MPC) has been widely used, as it can compute optimal trajectories and control inputs while explicitly enforcing safety constraints. 
Classical approaches to handling uncertainty in MPC include Robust MPC (RMPC) \cite{bemporad2007robust} and Stochastic MPC (SMPC) \cite{mesbah2016stochastic}. 
RMPC typically assumes a bounded uncertainty set and enforces safety constraints for the worst-case uncertainty realization, which can provide strong safety guarantees, but often leads to conservative decisions. SMPC reduces conservatism compared to RMPC by enforcing probabilistic safety constraints. However, probabilistic constraints cannot account for rare, safety-critical (``tail") events beyond the chosen probability threshold, and the safety depends on the accuracy of the assumed uncertainty distribution.

To address some of the limitations of RMPC and SMPC, risk-aware MPC (RA-MPC) \cite{singh2018framework} has recently gained attention, utilizing risk measures such as conditional value-at-risk (CVaR)\cite{rockafellar2000optimization} to enforce safety constraints that can account for tail events while reducing conservatism compared to the worst-case approach used in RMPC. 
However, existing RA-MPC methods still rely on accurate characterization of the underlying uncertainty to evaluate the risk and enforce safety constraints. 
In practice, the true uncertainty distribution is often unknown. 
Therefore, existing methods commonly model the uncertainty through empirical estimation or distributional assumptions, neither of which is guaranteed to accurately reflect the true uncertainty encountered during operation. 
As a result, these methods do not guarantee risk constraint satisfaction with respect to the true uncertainty distribution.

To address this challenge, we propose a novel distribution-free RA-MPC framework with statistical guarantees on risk constraint satisfaction. 
Specifically, we develop a distribution-free risk quantification framework by extending conformal risk control (CRC) \cite{angelopoulos2022conformal} to general spectral risk measures \cite{dowd2008spectral}. 
The proposed method generates prediction sets that directly control spectral risk below a user-specified threshold without requiring distributional assumptions.
We then incorporate the prediction sets into an RA-MPC framework to enforce spectral risk constraints with respect to the true uncertainty distribution. 

Our main contributions in this work are:
\begin{enumerate}
    \item We develop a distribution-free risk quantification method that generates prediction sets that control spectral risks below a user-specified threshold without requiring assumptions on the underlying uncertainty.
    \item We present an RA-MPC framework that incorporates these prediction sets to provide statistical safety guarantees in terms of spectral risk constraint satisfaction with respect to the true uncertainty.
    \item We validate the proposed MPC framework through simulation studies of stationary and dynamic obstacle avoidance, demonstrating improved safety and reduced solve time under uncertainty misspecification compared to a baseline RA-MPC framework.
\end{enumerate}

\section{Related Work}
RA-MPC has recently received significant attention in the autonomous planning and control domain. 
Existing works have explored leveraging a variety of risk measures and assumptions on the uncertainty distribution. 
For example, CVaR has been utilized to quantify collision risk with randomly moving obstacles in \cite{hakobyan2019risk}, where the uncertainty in obstacle location and orientation is estimated using sample average approximation (SAA). 
This approach has been extended to a distributionally robust formulation in \cite{hakobyan2021wasserstein}, improving robustness to estimation errors in the underlying uncertainty distribution. 
Furthermore, researchers have investigated using other risk measures beyond CVaR in the RA-MPC framework, such as entropic value-at-risk (EVaR) \cite{dixit2021risk} and more general $g$-entropic risk measures \cite{dixit2023risk}.

Existing RA-MPC methods generally assume that the uncertainty distribution is known, or can be accurately estimated from data. 
In practice, however, the true uncertainty distribution is often unknown and difficult to estimate. 
Therefore, existing methods cannot provide safety guarantees and may lead to unsafe decisions if the estimated or assumed uncertainty distribution is inaccurate. 

Conformal Prediction (CP) \cite{angelopoulos2023conformal,vovk2005algorithmic} has recently gained attention as a powerful distribution-free uncertainty quantification framework. 
CP generates prediction sets that provide coverage probability guarantees, without requiring assumptions on the underlying data distribution. 
CRC \cite{angelopoulos2022conformal} generalizes CP by controlling the expected value of bounded monotone loss functions below a user-specified threshold. 
Several extensions of CRC have been developed, including methods for handling covariate shift \cite{zecchin2025generalization,angelopoulos2022conformal} and optimized certainty equivalent (OCE) risk measures \cite{yeh2025conformal}. 
Closely related to our work, researchers extend CRC to distortion risk measures and provide high-probability risk bounds in \cite{chen2025conformal}. 
In contrast, our work provides direct control of spectral risk measures, and the underlying theory and application context are different.

Because of its distribution-free nature and robust statistical guarantees, CP and CRC have been widely applied in planning and control under uncertainty \cite{lindemann2024formal}. 
In the context of MPC, CP was used to quantify the uncertainty in obstacle trajectory predictions in \cite{lindemann2023safe}, where the prediction sets are incorporated into the MPC constraints to guarantee collision avoidance with user-specified probability. In \cite{dixit2023adaptive}, the prediction sets were adapted online to maintain collision avoidance guarantees under distribution shifts. 
In \cite{chee2023uncertainty,chee2024uncertainty}, CP was utilized to quantify uncertainty in learned dynamics models for MPC.
More closely related to our work, researchers in \cite{zecchin2024forking} utilize ideas from CRC to present an MPC framework with statistical guarantees for expectation constraints.
Additionally, researchers in \cite{zhang2025conformal,zhou2024safety,hsu2025statistical,yang2023safe} design control barrier functions (CBF) that incorporate uncertainty quantification from CP for safety-critical control. 
However, existing conformal methods in planning and control are limited to probabilistic or expectation-based constraints rather than the more general risk constraints considered in RA-MPC.

Based on the literature review, we identify the following research gap.
Existing RA-MPC methods can account for tail events, but they typically depend on estimations of, or simplifying assumptions about, the underlying uncertainty, which may be inaccurate. 
Conversely, existing conformal methods for uncertainty and risk quantification do not require such assumptions, but they do not extend to direct control of spectral risk measures. 
Additionally, existing RA-MPC methods are often limited to a small number of risk measures. 
This motivates the development of a distribution-free RA-MPC framework that can enforce general spectral risk constraints while remaining robust to uncertainty misspecification.
\label{sec:related_work}

\section{Preliminaries}
This section introduces the fundamental concepts and notation used for CRC and spectral risk measures.
\subsection{Conformal Risk Control}
CRC \cite{angelopoulos2022conformal} extends CP to generate prediction sets that provide guarantees to control the expected value of bounded, monotone loss functions below a user-specified threshold. 
Formally, given a loss function $L$ and a model that outputs a prediction $\hat{Y}$ of $Y$, CRC finds the tightest prediction set $\mathcal{C_{\lambda}}(\hat{Y})$ parameterized by $\lambda\in\Lambda$ that guarantees
\begin{equation}
    \mathbb{E}[L(\mathcal{C_{\lambda}}(\hat{Y}),Y)]\leq\theta
    \label{eq:crc}
\end{equation}
where $\theta$ is the user-specified risk threshold. 
The guarantees provided by CRC hold for bounded loss functions that are monotonically nonincreasing with $\lambda$. 
To find the value of $\lambda$ corresponding to the tightest $C_{\lambda}$ that satisfies \eqref{eq:crc}, CRC assumes access to a calibration dataset of exchangeable data points $D_{cal}:=\{Y_i\}_{i=1}^n$ and the corresponding predictions $\hat{Y}_i$. 
Then, $\hat{\lambda}$, the optimized prediction set parameter, can be found by solving the following optimization problem:
\begin{equation}
    \hat\lambda = \inf_{\lambda\in\Lambda}\biggl\{\lambda:\frac{1}{n+1}\sum^{n}_{i=1}L_{i}(\lambda)+\frac{B}{n+1}\leq\theta\biggl\}
    \label{eq:crc_opt}
\end{equation}
where $L_i(\lambda):=L(\mathcal{C_{\lambda}}(\hat{Y}_i),Y_i)$, and $B$ is the upper bound of $L_i(\lambda)$. See \cite{angelopoulos2022conformal} for the formal proof.

\subsection{Spectral Risk Measures}
Consider a probability space $(\Omega,\mathcal{F},P)$, where $\Omega,\mathcal{F},$ and $P$ are the sample space, $\sigma$-algebra, and the probability measure over $\mathcal{F}$, respectively. A random variable $Z:\Omega\rightarrow \mathbb{R}$ represents the cost of the samples in $\Omega$. 
Let $\mathcal{Z}$ be the set of random variables.
A \textit{risk measure} $\rho:\mathcal{Z}\rightarrow\mathbb{R}$ is a function that maps a random variable to a real number that captures how risky the possible outcomes of a random variable are.

Coherent risk measures \cite{artzner1999coherent} were first proposed in the financial literature to characterize risk measures that have desirable properties to accurately assess the risk associated with financial assets. 
Spectral risk measures are a subclass of coherent risk measures that also satisfy law invariance and comonotone additivity. Researchers have advocated for the use of spectral risk measures to assess risk in robotics in \cite{majumdar2019should}, providing intuitive explanations on how the properties of spectral risk measures lead to rational risk assessment.

\begin{definition}[Spectral risk measure]
    Consider two random variables $Z,Z' \in \mathcal{Z}$. A spectral risk measure is defined as a risk measure $\rho:\mathcal{Z} \to \mathbb{R}$ that satisfies the following properties:
    \begin{enumerate}
        \item Monotonicity: If $Z\leq Z'$, then $\rho(Z) \leq \rho(Z')$
        \item Subadditivity: $\rho(Z+Z')\leq \rho(Z)+\rho(Z')$
        \item Translation Invariance: $\rho(Z+c)=\rho(Z)+c$ $\forall c\in \mathbb{R}$
        \item Positive Homogeneity: $\rho(c Z)=c\rho(Z)$ $\forall c\geq 0$
        \item Law Invariance: If $Z$ and $Z'$ have the cumulative distribution functions (CDF) $F_Z$ and $F_{Z'}$, and if $F_Z=F_{Z'}$, then $\rho(Z)=\rho(Z')$
        \item Comonotone Additivity: If $Z$ and $Z'$ are comonotone, then $\rho(Z+Z')=\rho(Z)+\rho(Z')$
    \end{enumerate}
    \label{def:spectral_risk}
\end{definition}

Additionally, the spectral risk measure of a random variable $Z$ can be represented in the form
\begin{equation}
    \rho(Z)=\int_{0}^{1}\phi(v)F_Z^{-1}(v)dv
    \label{eq:spec_int}
\end{equation}
where $F_Z^{-1}(v):=\inf\{ p\in\mathbb{R}:F_Z(p)\geq v\},$ $\forall v\in[0,1]$ is the generalized inverse distribution function of $Z$, and $\phi(v)$ is a weighting function defined on $[0,1]$ that satisfies the following properties:
\begin{enumerate}
    \item $\phi(v)\geq0$
    \item $\int_0^1\phi(v)dv=1$
    \item $\phi(v)$ is nondecreasing.
\end{enumerate}

Spectral risk measures include commonly used risk measures such as CVaR and the Wang risk measure \cite{wang2000class}. 
For example, the weighting function for CVaR with risk level parameter $\alpha\in(0,1)$ is defined as
\begin{equation}
    \phi(v):= 
    \begin{cases}
        \frac{1}{1-\alpha}, & \text{if} \ v\in[\alpha,1] \\
        0, & \text{otherwise}
    \end{cases}
    \label{eq:4}
\end{equation}
and the weighting function for the Wang risk measure with risk level parameter $\gamma$ is defined as
\begin{equation}
    \phi(v):= \exp\biggl(\gamma\Phi^{-1}(v)-\frac{\gamma^2}{2}\biggl)
    \label{eq:5}
\end{equation}
where $\Phi$ is the CDF of the standard normal distribution.

\section{Conformal Spectral Risk Control}
In this section, we extend CRC to spectral risk measures by presenting a reformulation of \eqref{eq:spec_int} as a weighted expectation and utilizing ideas from weighted CRC (W-CRC).
\subsection{Reformulation of Spectral Risk Measures}
Since the form of spectral risk measures presented in \eqref{eq:spec_int} is difficult to directly incorporate into CRC, we present an equivalent reformulation.
\begin{proposition}[Probability Integral Transform\cite{angus1994probability}]
    Let $Z$ be a continuous random variable with a CDF $F_Z$. Then, the random variable $U=F_Z(Z)$ has a standard uniform distribution $U\sim \text{Unif}(0,1)$.
    \label{prop:pit}
\end{proposition}
\begin{theorem}
    For a continuous random variable $Z$, a spectral risk measure $\rho(Z)$ defined by the weighting function $\phi$ can be represented in the form
    \begin{equation}
        \rho(Z)=\mathbb{E}[Z\phi(F_{Z}(Z))]
        \label{eq:spec_expectation}
    \end{equation}
    \label{thm:spec_risk}
\end{theorem}
\vspace{-8mm}
\begin{proof}
    Using the form of spectral risk measures in \eqref{eq:spec_int},
    \begin{equation}
        \rho(Z) = \int_{0}^{1}\phi(v)F_Z^{-1}(v)dv = \mathbb{E}[\phi(V)F_Z^{-1}(V)]
        \label{eq:7}
    \end{equation}
    where $V\sim \text{Unif} (0,1)$. Let $U$ be defined as in Proposition \ref{prop:pit}. Then, since $V$ and $U$ have the same distribution
    \begin{equation}
        \mathbb{E}[\phi(V)F_Z^{-1}(V)] = \mathbb{E}[\phi(U)F_Z^{-1}(U)]
        \label{eq:8}
    \end{equation}
    Substituting $U=F_Z(Z)$
    \begin{equation}
        \rho(Z)=\mathbb{E}[\phi(F_Z(Z))F_Z^{-1}(F_Z(Z))]
        \label{eq:9}
    \end{equation}
    The proof is complete by the property of quantile functions, where $F_Z^{-1}(F_Z(Z))=Z$ almost surely.
\end{proof}

\subsection{Conformal Spectral Risk Control via W-CRC}
The CRC framework introduced in Section \ref{sec:related_work} cannot be directly applied to generate prediction sets that control the weighted expectation of \eqref{eq:spec_expectation}. 
Therefore, we utilize the W-CRC\cite{angelopoulos2022conformal,zecchin2025generalization} framework, which has been used to maintain valid risk control guarantees under covariate shift. 
The W-CRC framework states that if we have bounded weights $w(\hat{Y})$ and a loss function $L(\lambda):=L(\mathcal{C_{\lambda}}(\hat{Y}),Y)$, we can provide the following guarantee for a user-specified risk threshold $\theta$:
\begin{equation}
    \mathbb{E}[w(\hat{Y})L(\lambda)]\leq\theta
    \label{eq:10}
\end{equation}
Similar to the CRC framework, assuming access to a calibration dataset $D_{cal}:=\{Y_i\}_{i=1}^n$ and the corresponding predictions $\hat{Y}_i$, the optimized prediction set parameter $\hat{\lambda}$ is found through solving a modified version of \eqref{eq:crc_opt}:
\begin{equation}
    \hat{\lambda}(\hat{Y}_{n+1}) = \inf_{\lambda\in \Lambda}\bigg\{\lambda:\frac{\sum_{i=1}^nw_iL_i(\lambda)+w_{n+1}B}{\sum_{i=1}^nw_i+w_{n+1}}\leq\theta\bigg\}
    \label{eq:wcrc_opt}
\end{equation}
where $w_i:=w(\hat{Y}_i)$. Refer to \cite{angelopoulos2022conformal,tibshirani2019conformal} for a formal proof. 
Therefore, we can solve the optimization problem in \eqref{eq:wcrc_opt} to control the expectation form of spectral risk measures presented in \eqref{eq:spec_expectation} by defining the weights as $w_i:=\phi(F_L(L_i(\lambda)))$.

Calculation of the weights $w_i$ requires $F_L$, which is the CDF of the loss function. 
In practice, since the true CDF is unknown, we can utilize an estimator $\hat{F}_L$. 
To minimize overfitting to the calibration data, we use an independent, exchangeable data set ${D}_{\text{est}}$ to construct the estimator. 
A common example of a CDF estimator is the empirical cumulative distribution function (eCDF), which we utilize in this work. 
The simulation studies in Section \ref{sec:simulation} validate that utilizing $\hat{F}_L$ still generates valid prediction sets.

Additionally, the optimization problem in \eqref{eq:wcrc_opt} includes $w_{n+1}$, which is the weight corresponding to the test data point $\hat{Y}_{n+1}$. 
Therefore, we must solve the optimization problem for every test data point. 
This adds significant online computation compared to CRC, which is problematic when incorporating prediction sets into real-time planning and control algorithms. 
Additionally, the W-CRC framework requires bounded weights, but in general, $\phi$ is unbounded. 
To address these problems, we introduce a conservative approximation in which $w_{n+1}$ is replaced by a finite upper bound $w_{\text{max}}$. 
In the case $\phi$ is bounded, then $w_{\text{max}}=\phi(1)$, as $\phi$ is nondecreasing. 
For unbounded $\phi$, we truncate $\phi$ at a value close to 1, then we add a correction factor to the risk threshold $\theta$ to account for the truncation. We first define the truncated risk measure:

\begin{definition}
    Consider a spectral risk measure $\rho$ with unbounded weighting function $\phi$. 
    Given a small constant $\epsilon>0$, the truncated risk measure $\rho_{t}$ is defined as
    \begin{equation}
        \rho_t(Z)=\int_0^1 \hat{\phi}(v)F_Z^{-1}(v)dv,
        \label{eq:12}
    \end{equation}
    \begin{equation}
        \hat{\phi}(v)=
        \begin{cases}
            \phi(v) \quad &\text{if} \ v\leq 1-\epsilon \\
            \phi(1-\epsilon) \quad &\text{if} \ v>1-\epsilon
        \end{cases}
        \label{eq:13}
    \end{equation}
\end{definition}
\begin{remark}
    In general, $\rho_t$ is not a spectral risk measure, since $\int_0^1\hat{\phi}(v)dv=1$ fails to hold. However, for a continuous random variable $Z$, Theorem \ref{thm:spec_risk} still holds, such that $\rho_t(Z)=\mathbb{E}(Z\hat{\phi}(F_Z(Z)))$.
    \label{rmk:trunc_risk}
\end{remark}

\begin{theorem}
    For a given spectral risk measure $\rho$ with unbounded weighting function $\phi$, a test data point $Y$ and the corresponding prediction $\hat{Y}$, and prediction set parameter $\lambda$, define the truncated risk measure $\rho_t$ based on Definition 2. Then for any risk threshold $\theta$,
    \begin{equation}
        \rho_t(L(\lambda))\leq\theta-\zeta \implies \rho(L(\lambda)) \leq \theta
        \label{eq:14}
    \end{equation}
    where $L(\lambda):=L(\mathcal{C_{{\lambda}}}(\hat{Y}),Y)$. The correction factor $\zeta:=B(\int_{1-\epsilon}^1\phi(v)dv-\epsilon w_{\text{max}})$, where $w_{\text{max}}:=\hat{\phi}(1)$ and $B$ is defined as the upper bound of $L$.
    \label{thm:spec_bound}
 \end{theorem}
 \begin{proof}     
    Using the form of spectral risk measures from \eqref{eq:spec_int} and Definition 2,
    \begin{align}
        \rho(L)&=\int_{0}^{1-\epsilon}F_L^{-1}(v)\phi(v)dv+\int_{1-\epsilon}^{1}F_L^{-1}(v)\phi(v)dv
        \label{eq:15}\\
        \rho_t(L)&=\int_{0}^{1-\epsilon}F_L^{-1}(v)\phi(v)dv+\int_{1-\epsilon}^{1}F_L^{-1}(v)\phi(1-\epsilon)dv
        \label{eq:16}
    \end{align}    
    where $\lambda$ is omitted for notational simplicity. Since $w_{\text{max}}=\hat{\phi}(1)=\phi(1-\epsilon)$ and $B$ is the upper bound of $L$,
    \begin{align}
        \rho(L)&=\rho_t(L)+\int_{1-\epsilon}^{1}F_L^{-1}(v)\phi(v)dv - \int_{1-\epsilon}^{1}F_L^{-1}(v)w_{\text{max}}dv \notag\\
        &\leq\rho_t({L})+\int_{1-\epsilon}^{1}B\phi(v)dv-\int_{1-\epsilon}^{1}Bw_{\text{max}}dv \notag\\
        & = \rho_t(L)+B\biggl(\int_{1-\epsilon}^{1}\phi(v)dv-\epsilon w_{\text{max}}\biggl)
        \label{eq:17}
    \end{align}
    Therefore, since $\zeta=B(\int_{1-\epsilon}^1\phi(v)dv-\epsilon w_{\text{max}})$,
    \begin{equation}
        \rho(L)\leq\rho_t(L)+\zeta
        \label{eq:spec_upper_bound}
    \end{equation}
    Because \eqref{eq:spec_upper_bound} holds regardless of $\lambda$, we can conclude that any $\lambda$ that controls $\rho_t$ below $\theta-\zeta$ also controls $\rho$ below $\theta$.
\end{proof}
    
Since $\rho_t$ can also be written as a weighted expectation (see Remark 1), and $\hat{\phi}$ is bounded, we can find the data-independent $\hat{\lambda}$ that satisfies the conditions of Theorem \ref{thm:spec_bound}. To do this, we solve the following optimization problem given $D_{\text{cal}}=\{Y_i\}_{i=1}^n$:
\begin{equation}
\hat{\lambda}=\inf_{\lambda \in \Lambda}\bigg\{\lambda:\frac{\sum_{i=1}^nw_iL_i(\lambda)+w_{\text{max}}B}{\sum_{i=1}^nw_i+w_{\text{max}}}\leq\theta-\zeta\bigg\}
\label{eq:19}
\end{equation}
where $w_i:=\hat{\phi}(F_L(L_i(\lambda)))$. 
As stated previously, in practice, we would utilize an estimator $\hat{F}_L$ designed with an independent, exchangeable dataset.
Algorithm 1 gives the overall algorithm for solving \eqref{eq:19}, which we refer to as conformal spectral risk control (CSRC).

\begin{algorithm}[t]
\caption{Conformal Spectral Risk Control (CSRC)}\label{alg:alg1}
\begin{algorithmic}[1]
    \footnotesize
    \STATE \textbf{Require:} Calibration set $D_{\text{cal}}$, estimation set $D_{\text{est}}$, loss function $L(\lambda)$, upper bound $B$, risk threshold $\theta$, parameter space $\Lambda=[\lambda_{\text{min}},\lambda_{\text{max}}]$, weighting function $\phi$, truncation limit $\epsilon$, parameter search step $\eta$
    \STATE $\lambda \gets \lambda_{\text{min}}$
    \IF{$\phi$ is bounded}
        \STATE $w_{\text{max}}\gets\phi(1), \ \zeta\gets0$ 
        \ELSE 
        \STATE $w_{\text{max}}\gets\phi(1-\epsilon), \ \zeta\gets B(\int_{1-\epsilon}^1\phi(v)dv-\epsilon w_{\text{max}})$
    \ENDIF
    \STATE Construct CDF estimator $\hat{F}_L$ using $D_{\text{est}}$
    \WHILE{$\lambda\leq\lambda_{\text{max}}$}
    \STATE $w_i \gets \phi(\hat{F}_L(L_i(\lambda)))$ for each ${L_i(\lambda)}_{i=1}^n$ from $D_{\text{cal}}$
    \STATE $R(\lambda)\gets\frac{\sum_{i=1}^nw_iL_i(\lambda)+w_{\text{max}}B}{\sum_{i=1}^nw_i+w_{\text{max}}}$
    \IF{$R(\lambda)\leq\theta-\zeta$}
        \STATE $\hat{\lambda} \gets \lambda$
        \RETURN $\hat{\lambda}$
    \ENDIF
    \STATE $\lambda \gets \lambda + \eta$
    \ENDWHILE
    \STATE $\hat{\lambda}\gets\lambda_{\text{max}}$
    \RETURN $\hat{\lambda}$
\end{algorithmic}
\end{algorithm}

\section{Distribution-Free RA-MPC with CSRC}
In this section, we present the distribution-free RA-MPC framework with spectral risk constraints that incorporates prediction sets generated by CSRC. Specifically, we formulate the following optimal control problem (OCP):
\begin{subequations}
\label{eq:ocp_original}
\begin{align}
    \min_{u} \ &\ J=\sum_{k=t}^{t+T}r(\xi_k,u_k) \label{eq:20a} \\
    \text{s.t.} \ &\ \xi_{k+1} = f(\xi_k,u_k) \label{eq:20b}\\
    &\ \rho(h(\xi_k,Y_{k}))\leq 0,\  k\in[t+1,t+T]
    \label{eq:risk_constraint_original}\\
    & \ \xi_{k} \in \Xi, \ u_k \in \mathcal{U} \label{eq:20d}
\end{align}
\end{subequations}
where $\xi_k$ is the state, $u_k$ is the control input, $Y_k\in\mathcal{Y}$ is an uncertain observation at time $k$, $r$ is the stage cost, $f$ is the system dynamics, $h$ is the constraint function, and $\rho$ is a spectral risk measure. 
In practice, however, the true distribution of $Y_k$ is often unknown. 
Instead, the controller only has access to predictions or estimates of $Y_k$, which may be misspecified. 
Therefore, we reformulate the OCP to incorporate the prediction sets generated by CSRC to provide statistical guarantees on satisfaction of the original spectral risk constraints with respect to the true distribution of $Y_k$. 
To do this, we introduce the following assumption:

\begin{assumption}
    The constraint function $h(\xi_k,Y_k)$ is Lipschitz continuous; i.e., for any $Y_k,\hat{Y}_k\in\mathcal{Y}$ and $\xi_k\in\Xi$ there exists a constant $K>0$ such that for some metric $d$ on $\mathcal{Y}$
\begin{equation}
    |h(\xi_k,Y_{k})-h(\xi_k,\hat{Y}_{k})| \leq Kd(\hat{Y}_{k},Y_{k})
    \label{eq:lipschitz}
\end{equation}
\label{assump:lipschitz}
\end{assumption}

\begin{figure}[t]
    \centering
    \includegraphics[]{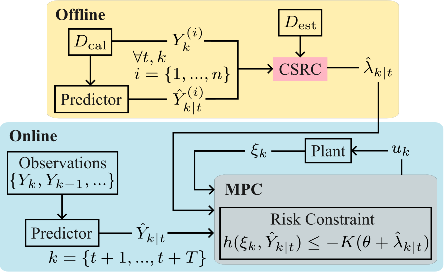}
    \caption{The overall CSRC-MPC framework. The optimized prediction set parameter $\hat{\lambda}_{k|t}$ found offline with CSRC is incorporated into the MPC risk constraints along with online predictions.}
    \label{fig:CSRC-MPC}
\end{figure}

\vspace{-10mm}
\begin{theorem}[CSRC-MPC]
    Let $\hat{Y}_{k|t}\in\mathcal{Y}$ be the prediction of $Y_k$ made at time $t$ and $\mathcal{C_{\hat{\lambda}}}(\hat{Y}_{k|t})$ be the optimized prediction set parameterized by $\hat{\lambda}$ generated through the CSRC algorithm using the following loss function:
    \begin{equation}
        L(\mathcal{C_{\lambda}}(\hat{Y}_{k|t}),Y_{k}) = d(\hat{Y}_{k|t},Y_{k})-\lambda_{k|t}
        \label{eq:loss_func}
    \end{equation}
    Then, for any constraint function $h$ that satisfies Assumption \ref{assump:lipschitz}, the OCP in \eqref{eq:ocp_original} can be conservatively reformulated as
    \begin{subequations}
    \begin{align}
    \min_{u} &\ J=\sum_{k=t}^{t+T}r(\xi_k,u_k) \\
    \text{s.t.} &\ \xi_{k+1} = f(\xi_k,u_k) \\
    &\ h(\xi_k,\hat{Y}_{k|t})\leq -K(\theta+\hat{\lambda}_{k|t}), \ k\in [t+1,t+T]
    \label{eq:risk_constraint_reformulated}\\
    & \ \xi_k \in \Xi, \ u_k \in \mathcal{U}
    \end{align}
    \label{eq:ocp_reformulated}
    \end{subequations}
\end{theorem}
\vspace{-8mm}
\begin{proof}
    Following \eqref{eq:lipschitz} in Assumption \ref{assump:lipschitz},
    \begin{equation}
        h(\xi_{k},Y_{k}) \leq h(\xi_k,\hat{Y}_{k|t})+Kd(\hat{Y}_{k|t},Y_{k})
        \label{eq:lipschitz_result}
    \end{equation}
    Substituting the loss function from \eqref{eq:loss_func} for the optimized prediction set generated by CSRC into \eqref{eq:lipschitz_result}
    \begin{equation}
        h(\xi_{k},Y_{k}) \leq h(\xi_k,\hat{Y}_{k|t})+K(L(\mathcal{C_{\hat{\lambda}}}(\hat{Y}_{k|t}),Y_{k})+\hat{\lambda}_{k|t})
    \end{equation}
    Taking the spectral risk of $h(\xi_{k},Y_{k})$ over the uncertainty in $Y_{k}$ and using the properties in Definition \ref{def:spectral_risk}, we get
    \begin{equation}
        \rho(h(\xi_{k},Y_{k})) \leq h(\xi_k,\hat{Y}_{k|t})+K(\rho(L(\mathcal{C_{\hat{\lambda}}}(\hat{Y}_{k|t}),Y_{k})) +\hat{\lambda}_{k|t})
    \end{equation}
    From Theorem \ref{thm:spec_bound}, the optimized prediction set $\mathcal{C_{\hat{\lambda}}} (\hat{Y}_{k|t})$ satisfies $\rho(L(\mathcal{C_{\hat{\lambda}}}(\hat{Y}_{k|t}),Y_{k})) \leq \theta$. Therefore, we can conclude
    \begin{equation}
        \rho(h(\xi_{k},Y_{k})) \leq h(\xi_k,\hat{Y}_{k|t}) +K(\theta+\hat{\lambda}_{k|t})
    \end{equation}
    This shows that any solution satisfying the reformulated constraint in \eqref{eq:risk_constraint_reformulated} also satisfies the original spectral risk constraint in \eqref{eq:risk_constraint_original}.
\end{proof}

Fig. \ref{fig:CSRC-MPC} illustrates the overall pipeline of incorporating the offline CSRC prediction set parameterized by $\hat{\lambda}_{k|t}$ into the MPC framework to enforce risk constraints online with respect to the true uncertainty distribution.

\section{Simulation Studies}
In this section, we present simulation results to validate the CSRC framework for multiple spectral risk measures and demonstrate the robustness of the proposed MPC framework under uncertainty misspecification. 
To do this, we compare the proposed framework with a baseline RA-MPC framework in two 2D obstacle avoidance scenarios. 
In the first scenario, we consider stationary obstacle avoidance with a misspecified obstacle location estimation model. 
In the second scenario, we consider dynamic obstacle avoidance with a misspecified obstacle trajectory prediction model.
\label{sec:simulation}

\begin{figure}[t]
\centering
\includegraphics[width=1\linewidth]{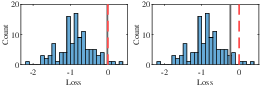}
\caption{Histograms of loss function evaluations on $D_{\text{test}}$ for stationary obstacle avoidance. The red line represents the user-specified risk level, and the gray line represents the empirical risk. Left: CVaR with risk parameter $\alpha=0.9$. Right: Wang risk measure with risk parameter $\gamma=1.5$.}
\label{fig:risk_control_stationary}
\end{figure}

\subsection{Simulation Setup}
For the simulations, we use the kinematic bicycle model \cite{polack2017kinematic} with the state vector $\xi$. The dynamics are given as
\begin{equation}
    \dot{\xi}=
    \begin{bmatrix}
        \dot{x} \\ \dot{y} \\ \dot{\psi} \\ \dot{v_x} \\ \dot{\delta} \\ \dot{a_x}
    \end{bmatrix} =
    \begin{bmatrix}
        v_x\cos(\psi+\beta) \\
        v_x\sin(\psi+\beta) \\
        v_x\frac{\cos(\beta)\tan(\delta)}{l_f+l_r} \\
        a_x \\
        \dot{\delta} \\
        j_x
    \end{bmatrix}
\end{equation}
where $\beta=\tan^{-1}(l_r/(l_f+l_r)\tan(\delta))$, $(x,y)$ represents the position of the vehicle center of mass, $\psi$, $\delta$ represent the yaw and steering angle of the vehicle, and $v_x$, $a_x$ represent the longitudinal velocity and acceleration. 
The variables $l_r$ and $l_f$ are the distances from the center of mass to the rear and front wheel axles. 
The control inputs are the steering rate $\dot{\delta}$ and the longitudinal jerk $j_x$.

To solve the OCP in \eqref{eq:ocp_reformulated}, we utilize Model Predictive Path Integral (MPPI) control, which is a sampling-based MPC algorithm. Therefore, the constraints are incorporated into the cost. The cost function is
\begin{equation}
    r(\xi_k,u_k,Y_k):=c_{\text{con}}+c_{\text{stage}}
\end{equation}
where $c_{\text{con}}$ penalizes constraint violations and is defined as
\begin{multline}
    c_{\text{con}}:=w_\xi\mathbbm{1}\{\xi_k\notin\Xi\}+w_u\mathbbm{1}\{u_k\notin\mathcal{U}\}\\
    +w_{\text{risk}}\mathbbm{1}\{h(\xi_k,\hat{Y}_{k|t})>-K(\theta+\hat{\lambda}_{k|t})\}
\end{multline}
where the tunable weights $w_\xi$, $w_u$, $w_{\text{risk}}$ are each set to $2\times10^6$. $c_{\text{stage}}$ penalizes distance to goal and large control inputs:
\begin{equation}
    c_{\text{stage}}:=w_{g}\sqrt{(x_k-x_g)^2+(y_k-y_g)^2}+w_{j_x}j_x^2+w_{\dot{\delta}}\dot{\delta}^2
\end{equation}
where $(x_g,y_g) =(20,0)$ is the goal position. The tunable weights $w_{j_x}$ and $w_{\dot{\delta}}$ are each set to $100$, and $w_g$ is set to $1000$.
The MPC prediction horizon is set to $T=10$ time steps with a sampling time of $T_s=\qty{0.1}{s}$, and each simulation is run for $N=55$ time steps. 
1024 trajectory rollouts are used for the MPPI optimization. 
For both simulation scenarios, the obstacles are assumed to be circular with a known and fixed radius $R=\qty{2}{m}$.

\begin{figure}[t]
    \centering
    \includegraphics[width=1\linewidth]{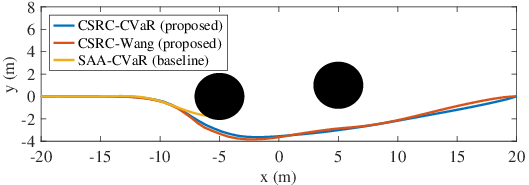}
    \caption{Representative vehicle trajectories for stationary obstacle avoidance, comparing the proposed CSRC-MPC using CVaR and Wang risk measure with the SAA-MPC baseline using the CVaR risk measure.}
    \label{fig:stationary_traj}
\end{figure}

\subsection{Stationary Obstacle Avoidance}
In this scenario, the vehicle is required to avoid two stationary obstacles. 
At each simulation timestep, the vehicle receives an estimate of the location of each obstacle. 
Let $(x_{\text{true}}^{(m)}, \ y_{\text{true}}^{(m)})$ represent the true location of the $m^\text{th}$ obstacle. 
The vehicle assumes that the obstacle location estimate at time $t$ is modeled as $\hat{Y}_t^{(m)}:=(x_{\text{true}}^{(m)}+\delta \hat{x}, \ y_{\text{true}}^{(m)}+\delta \hat{y})$, where the assumed estimation error $\delta \hat{x},\delta \hat{y} \sim \mathcal{N}(1,0.3^2)$. 
We consider a setting where the assumed estimation error is misspecified and does not reflect the true, unknown error distribution. 
Specifically, we model the true obstacle location estimate as ${Y^{(m)}}:=(x_{\text{true}}^{(m)}+\delta x, \ y_{\text{true}}^{(m)}+\delta y)$, where the true estimation error $\delta x,\delta y \sim \mathcal{N}(0,0.3^2)$. 
Since the true estimation error model is unknown to the vehicle at test time, prediction sets generated by the CSRC framework are used to enforce risk constraints defined with respect to the true estimation error model. 
The loss function used for the CSRC framework is
\begin{equation}
    L(\lambda) = ||Y-\hat{Y}_t||-\lambda
    \label{eq:loss_stationary}
\end{equation}
with user-specified risk threshold $\theta=0$. 
For circular obstacles, the prediction set parameterized by $\lambda$ can be interpreted as a buffer added to the obstacle radius for a given estimation $\hat{Y}_t$ to account for the misspecification between $\hat{Y}_t$ and $Y$. 
Since the Gaussian distribution has unbounded support, whereas the CSRC framework requires the loss function to be bounded, we constrain $Y$ and $\hat{Y}_t$ such that $B=5$. 
Based on this loss function and risk threshold, the obstacle constraint function is
\begin{equation}
    h(\xi_k,\hat{Y}_t^{(m)})=-||z_k-\hat{Y}_t^{(m)}||+R\leq -\hat{\lambda}, \quad m=1,2
\end{equation}
where $z_k:=(x_k,y_k)$ represents the vehicle position, and the Lipschitz constant $K=1$ for the distance function. 

We collect 1210 data points for this scenario and divide them into calibration, estimation, and test data sets: ${D}_{\text{cal}}=\{Y_i\}_{i=1}^{1000}$, ${D}_{\text{est}}=\{Y_i\}_{i=1001}^{1100}$, and ${D}_{\text{test}}=\{Y_i\}_{i=1101}^{1210}$ along with their corresponding assumed location estimates $\hat{Y}_i$. 
The calibration and estimation datasets are generated using randomized true obstacle locations, whereas the test dataset is equally divided between fixed obstacle locations $(x^{(1)}_{\text{true}},y^{(1)}_{\text{true}})=(-5,0)$ and $(x^{(2)}_{\text{true}},y^{(2)}_{\text{true}})=(5,1)$ used in the simulation. 
The CSRC framework is evaluated with two spectral risk measures: CVaR and Wang. 
The risk levels used for CVaR and Wang risk measures are $\alpha=0.9$ and $\gamma=1.5$, respectively, both of which represent high levels of risk aversion. 
Using Algorithm 1, the optimized prediction set parameters are found to be $\hat{\lambda}_{\text{CVaR}}=2.25$ and $\hat{\lambda}_{\text{Wang}}=2.36$. 
Fig. \ref{fig:risk_control_stationary} shows the histograms of the risks of the loss function \eqref{eq:loss_stationary} evaluated on $D_{\text{test}}$. 
For both CVaR and Wang, the risk is controlled below the user-specified threshold of $\theta=0$, validating the CSRC framework. We repeated the simulation with $D_{\text{est}}$ size increased to 1000, and observed no meaningful improvement in the risk control performance. This suggests that in practice, estimation of $F_L$ with a relatively small dataset still leads to accurate prediction set calibration.

\begin{figure}[t]
    \centering
    \includegraphics[width=1\linewidth]{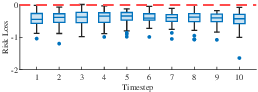}
    \caption{Boxplot of risk of the loss function by prediction timesteps evaluated on $D_{\text{test}}$ for the dynamic obstacle avoidance scenario. CVaR with risk parameter $\alpha=0.9$ is used for the spectral risk measure.}
    \label{fig:dynamic_coverage}
\end{figure}

Fig. \ref{fig:stationary_traj} shows a representative vehicle trajectory with the proposed CSRC-MPC framework, where the vehicle receives one estimation of each obstacle location $\hat{Y}_t^{(1)},\hat{Y}_t^{(2)}\in D_{\text{test}}$ at every simulation timestep. 
We compare the proposed framework with the baseline SAA-MPC\cite{hakobyan2019risk}, which uses online sampling to approximate the risk constraint. 
We assume that SAA-MPC has access to 100 samples of $\hat{Y}_t^{(1)}$ and $\hat{Y}_t^{(2)}$ at each simulation timestep, and the risk level is chosen to be the same as the CSRC framework with $\alpha=0.9$. 
Fig. \ref{fig:stationary_traj} shows that the baseline SAA-MPC fails to avoid the obstacles even under a high level of risk aversion, as the misspecified estimation model leads to an inaccurate risk constraint approximation. 
In contrast, the proposed MPC framework avoids the obstacles safely for both CVaR and Wang risk measures. 
This improvement is because the CSRC framework calibrates the prediction sets based on the discrepancy between the assumed and true estimation error distributions, making the proposed MPC framework more robust to uncertainty misspecification.

\begin{figure*}[t]
    \centering
    \includegraphics[width=0.90\linewidth]{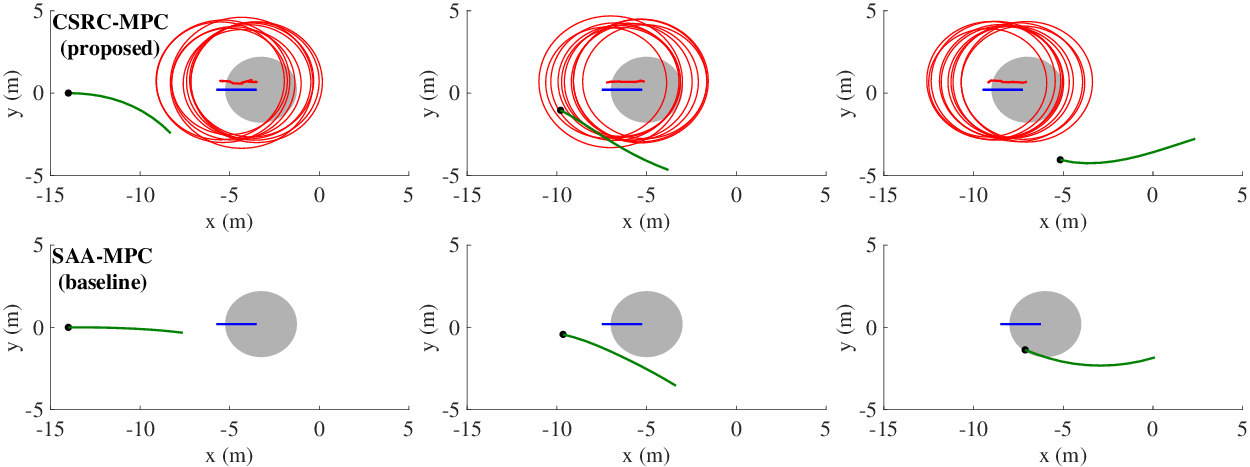}
    \caption{Representative vehicle trajectories for dynamic obstacle avoidance. Green represents the optimal planned trajectory, and blue represents the trajectory of the true obstacle. Top: Vehicle trajectory from CSRC-MPC using CVaR with risk parameter $\alpha=0.9$ at times $t=11,18$, and $26$. The obstacle trajectory predictions and their corresponding prediction sets are marked in red. Bottom: Vehicle trajectory from SAA-MPC using CVaR with risk parameter $\alpha=0.9$ at times $t=11,18,$ and $22$.}
    \label{fig:dynamic_traj}
\end{figure*}

\subsection{Dynamic Obstacle Avoidance}
In the second scenario, the vehicle is required to avoid a dynamic obstacle. 
Specifically, we consider an obstacle moving towards the vehicle in a straight trajectory with constant velocity. 
We assume the vehicle has access to a trajectory prediction model that predicts the obstacle trajectory for 10 future time steps using the 10 previous observations. 
For this simulation, we utilize a linear predictor based on singular spectrum analysis (SSA), similar to \cite{wei2022moving}.
However, the proposed framework is agnostic to the choice of prediction model. 
Similar to the previous scenario, we model the true observation at time $k$ as $Y_k:=(x_{\text{true}}^k+\delta x,y_{\text{true}}^k+\delta y)$, where $(x_{\text{true}}^k,y_{\text{true}}^k)$ is the true obstacle location, and $\delta x,\delta y \sim \mathcal{N}(0,0.05^2)$ represents the unknown true estimation error. 
To simulate a misspecified prediction model, we add an artificial bias of \qty{0.5}{m} to the predicted $x$ and $y$ positions of the obstacle. 
Let $\hat{Y}_{k|t}$ represent the obstacle location predictions made at time $t$ for future timesteps $k\in[t+1,t+T]$. 
Then, the corresponding loss function for the CSRC framework is
\begin{equation}
    L(\lambda)=||Y_{k}-\hat{Y}_{k|t}||-\lambda_{k|t}
    \label{eq:loss_dynamic}
\end{equation}
where $\lambda_{k|t}$ represents the prediction set parameter associated with $\hat{Y}_{k|t}$. 
As in the previous scenario, the risk threshold is set to $\theta=0$, and the observation and predictions are constrained to the simulation domain. 
Based on the loss function and risk threshold, the obstacle avoidance constraint is defined as
\begin{equation}
    h(\xi_k,\hat{Y}_{k|t})=-||z_k-\hat{Y}_{k|t}||+R\leq -\hat{\lambda}_{k|t}
\end{equation}
where $z_k=(x_k,y_k)$ is the vehicle position and $\hat{\lambda}_{k|t}$ is the optimized prediction set parameter.

We collect 2100 trajectory samples for this scenario using randomized initial obstacle locations, heading directions, and velocities. 
The data are then divided into calibration, estimation, and test sets: ${D}_{\text{cal}}=\{Y^{(i)}\}_{i=1}^{1000}$, ${D}_{\text{est}}=\{Y^{(i)}\}_{i=1001}^{1100}$, and ${D}_{\text{test}}=\{Y^{(i)}\}_{i=1101}^{2100}$. 
For simplicity, we assume past observations $(Y_{-10},...,Y_{-1})$ are available to the prediction model. 
Therefore, $Y^{(i)}:=(Y_{-10},Y_{-9},...)$. 
Then, the corresponding predictions $\hat{Y}^{(i)}_{k|t}$ are obtained for all times $t$ and $k\in[t+1,t+T]$. 
We evaluate the CSRC framework using CVaR with the risk level parameter $\alpha=0.9$. 
The box plot in Fig. \ref{fig:dynamic_coverage} shows that the average risk of the loss in \eqref{eq:loss_dynamic} is controlled below the user-specified threshold of $\theta=0$ at every prediction step, empirically validating the statistical guarantees provided by the CSRC framework.

We run 1000 simulations using the obstacle trajectories from $D_{\text{test}}$ and compare the proposed CSRC-MPC framework with the baseline SAA-MPC that uses 100 online trajectory prediction samples to estimate risk. The performance is evaluated using three metrics: obstacle constraint violation rate, success rate, and average solve time. 
The obstacle constraint violation rate is the percentage of simulation trajectories that have at least one obstacle constraint violation, where the risk is evaluated empirically with \qty{10000} Monte Carlo samples from the true observation model for each obstacle trajectory from $D_{\text{test}}$. 
The success rate is defined as the percentage of trajectories that avoid the obstacle successfully. 
Table \ref{tab:performance} shows that the proposed MPC framework significantly improves obstacle avoidance performance and risk constraint satisfaction. 
Moreover, because CSRC-MPC generates the prediction sets offline, the average solve time is lower compared to SAA-MPC and fast enough for real-time planning and control.

\begin{table}[t]
\caption{Dynamic Obstacle Avoidance Results}
    \centering
    \begin{tabular}{ccc}
    \toprule
     \textbf{Performance Metric}  & \textbf{SAA-MPC} & \textbf{CSRC-MPC} \\
     & (baseline) & (proposed) \\
     \midrule
     Obstacle constraint violation (\%)  & 52.9 & 6.0  \\
     Success rate (\%) & 67.5 & 100 \\ 
     Average solve time (ms) & 104.9 & 49.9\\
     \bottomrule
    \end{tabular}
    \label{tab:performance}
\end{table}

\begin{remark}
    Although CSRC-MPC reduces obstacle constraint violations compared to SAA-MPC, it does not eliminate constraint violations for every simulation. This is because the empirical risk for this metric is evaluated with Monte Carlo samples of each obstacle trajectory, whereas the CSRC framework provides marginal guarantees over all possible obstacle trajectories. However, the improvement over the baseline is still significant, as the constraint violations do not lead to collisions.
    \label{rmk:marginal}
\end{remark}

Fig. \ref{fig:dynamic_traj} compares representative vehicle trajectories generated by the proposed CSRC-MPC and the baseline SAA-MPC, with both frameworks using CVaR with risk parameter $\alpha=0.9$. 
The CSRC framework generates prediction sets that allow the vehicle to avoid the obstacle safely despite misspecification in the obstacle trajectory predictions. 
In contrast, SAA-MPC relies on risk estimation solely based on online sampling from the biased prediction model, which leads to collisions even under high levels of risk aversion.

\section{Conclusion}
We present CSRC, a novel risk quantification framework that generates prediction sets that directly control spectral risk measures below a user-specified threshold without requiring assumptions on the underlying uncertainty distribution. 
We develop CSRC-MPC, an RA-MPC framework that incorporates these prediction sets to provide statistical safety guarantees in terms of spectral risk constraint satisfaction. 
Simulation results for stationary obstacle avoidance demonstrate valid risk control for CVaR and Wang risk measures, and show improved obstacle avoidance performance compared to the baseline SAA-MPC under a misspecified estimation error model. 
Results for the dynamic obstacle avoidance scenario further demonstrate improved obstacle avoidance performance compared to the baseline under a misspecified obstacle trajectory prediction model. 
Additionally, the proposed framework reduces the solve time compared to the baseline, demonstrating real-time feasibility without sacrificing safety.

Future work may incorporate ideas from adaptive CP \cite{gibbs2021adaptive} to maintain valid guarantees under distribution shift. 
Also, as noted in Remark \ref{rmk:marginal}, the guarantees provided by the CSRC framework are marginal. 
It would be interesting to investigate providing stronger conditional guarantees, especially in more safety-critical scenarios. 
Other potential future research directions include extending the framework to account for multiple sources of risk and validating the proposed MPC framework through real-world hardware experiments.

\bibliographystyle{IEEEtran}
\bibliography{references}

\end{document}